\documentclass{article}

\usepackage{amsmath,amsfonts,amsthm,amssymb,stmaryrd,xcolor}
\usepackage[bookmarks=true, colorlinks=true, linkcolor=blue!50!red, citecolor=orange,
pdfencoding=unicode]{hyperref}
\usepackage{tikz}
\usepackage{onecolceurws}
\usepackage{enumitem}
\usepackage{comment}
\usepackage{mathabx}

\input diagxy  

\makeatletter
\def\copyrightspace{
\long\def\@makefntext##1{\noindent ##1}
\footnotesep 1em
\footnotetext[0]{\em Copyright \copyright\ by the paper's authors.
Copying permitted for private and academic purposes.}
\footnotetext[0]{In: J.\ Cheney, H-S.\ Ko (eds.): Proceedings of the Eighth International Workshop on Bidirectional Transformations (Bx 2019),
Philadelphia, PA, USA, June 4, 2019, published at http://ceur-ws.org}
}
\makeatother

\newcommand{\define}[1]{\textbf{#1}}

\newcommand{\rr}{\mathbb{R}}
\newcommand{\prl}{\Vert}
\newcommand{\cp}{\mathbin{\fatsemi}}
\newcommand{\id}{\mathrm{id}}

\newcommand{\set}{\mathsf{Set}}
\newcommand{\alens}{\mathsf{Lens}}
\newcommand{\wbalens}{\mathsf{WBLens}}
\newcommand{\slens}{\mathsf{SLens}}
\newcommand{\lrn}{\mathsf{Learn}}
\newcommand{\equivf}{f}

\usetikzlibrary{
	cd,
	shapes.geometric,
	decorations.markings,
	decorations.pathmorphing,
	positioning,
	arrows,
	shapes,
	calc,
	fit,
	quotes}

\tikzset{
   oriented WD/.style={
      every to/.style={out=0,in=180,draw},
      label/.style={
         font=\everymath\expandafter{\the\everymath\scriptstyle},
         inner sep=0pt,
         node distance=2pt and -2pt},
      semithick,
      node distance=1 and 1,
      decoration={markings, mark=at position .5 with {\arrow{stealth};}},
      ar/.style={postaction={decorate}},
      execute at begin picture={\tikzset{
         x=\bbx, y=\bby,
         every fit/.style={inner xsep=\bbx, inner ysep=\bby}}}
      },
   bbx/.store in=\bbx,
   bbx = 1.5cm,
   bby/.store in=\bby,
   bby = 1.75ex,
   bb port sep/.store in=\bbportsep,
   bb port sep=2,
   bb port length/.store in=\bbportlen,
   bb port length=0pt,
   bb penetrate/.store in=\bbpenetrate,
   bb penetrate=0pt,
   bb min width/.store in=\bbminwidth,
   bb min width=1cm,
   bb rounded corners/.store in=\bbcorners,
   bb rounded corners=5pt,
   bb small/.style={bb port sep=1, bb port length=2.5pt, bbx=.4cm, bb min width=.4cm, bby=.7ex},
   bb/.code 2 args={
      \pgfmathsetlengthmacro{\bbheight}{\bbportsep * (max(#1,#2)) * \bby}
      \pgfkeysalso{draw,minimum height=\bbheight,minimum width=\bbminwidth,outer sep=0pt,
         rounded corners=\bbcorners,thick,
         prefix after command={\pgfextra{\let\fixname\tikzlastnode}},
         append after command={\pgfextra{\draw
            \ifnum #1=0{} \else foreach \i in {1,...,#1} {
               ($($(\fixname.north
	       west)+(0,.9\bbportsep)$)!{\i/(#1+1)}!($(\fixname.south
	       west)-(0,.9\bbportsep)$)$)
	       +(-\bbportlen,0) coordinate (\fixname_in\i) -- +(\bbpenetrate,0) coordinate (\fixname_in\i')}\fi 
            \ifnum #2=0{} \else foreach \i in {1,...,#2} {
               ($($(\fixname.north
	       east)+(0,\bbportsep)$)!{\i/(#2+1)}!($(\fixname.south
	       east)-(0,\bbportsep)$)$) +(-\bbpenetrate,0) coordinate (\fixname_out\i') -- +(\bbportlen,0) coordinate (\fixname_out\i)}\fi;
         }}}
   },
   bb name/.style={append after command={\pgfextra{\node[anchor=north] at
(\fixname.north) {#1};}}},
   ibb port sep/.store in=\ibbportsep,
   ibb port sep=2,
   ibb port length/.store in=\ibbportlen,
   ibb port length=4pt,
   ibb min width/.store in=\ibbminwidth,
   ibb min width=1cm,
   ibb rounded corners/.store in=\ibbcorners,
   ibb rounded corners=1pt,
   ibb small/.style={ibb port sep=1, ibb port length=2.5pt, bbx=.4cm, ibb min width=.4cm, bby=.7ex},
   ibb/.code 2 args={
	   \pgfmathsetlengthmacro{\ibbheight}{\ibbportsep * (max(#1,#2)) * \bby}
	   \pgfkeysalso{draw,color=gray!50,minimum height=\ibbheight,minimum width=\ibbminwidth,outer sep=0pt,
		   rounded corners=\ibbcorners,thick,
		   prefix after command={\pgfextra{\let\fixname\tikzlastnode}},
		   append after command={\pgfextra{\coordinate
			   \ifnum #1=0{} \else foreach \i in {1,...,#1} {
				   ($($(\fixname.north
					west)+(0,.9\ibbportsep)$)!{\i/(#1+1)}!($(\fixname.south
						west)-(0,.9\ibbportsep)$)$)
					   +(-\ibbportlen,0) coordinate (\fixname_in\i) -- +(\ibbportlen,0) coordinate (\fixname_in\i')}\fi 
					   \ifnum #2=0{} \else foreach \i in {1,...,#2} {
						   ($($(\fixname.north
							east)+(0,\ibbportsep)$)!{\i/(#2+1)}!($(\fixname.south
								east)-(0,\ibbportsep)$)$) +(-\ibbportlen,0) coordinate (\fixname_out\i') -- +(\ibbportlen,0) coordinate (\fixname_out\i)}\fi;
		   }}}
   },
   ibb name/.style={append after command={\pgfextra{\node[anchor=north] at
   (\fixname.north) {#1};}}},
   blankbb port sep/.store in=\blankbbportsep,
   blankbb port sep=2,
   blankbb min width/.store in=\blankbbminwidth,
   blankbb min width=1cm,
   blankbb rounded corners/.store in=\blankbbcorners,
   blankbb rounded corners=1pt,
   blankbb small/.style={blankbb port sep=1, blankbb port length=2.5pt, bbx=.4cm, blankbb min width=.4cm, bby=.7ex},
   blankbb/.code 2 args={
	   \pgfmathsetlengthmacro{\blankbbheight}{\blankbbportsep * (max(#1,#2)) * \bby}
	   \pgfkeysalso{draw,color=gray!50,minimum height=\blankbbheight,minimum width=\blankbbminwidth,outer sep=0pt,
		   rounded corners=\blankbbcorners,thick,
		   prefix after command={\pgfextra{\let\fixname\tikzlastnode}},
		   append after command={\pgfextra{\draw
			\ifnum #1=0{} \else foreach \i in {1,...,#1} {
			   ($($(\fixname.north
			   west)+(0,.9\ibbportsep)$)!{\i/(#1+1)}!($(\fixname.south
			   west)-(0,.9\ibbportsep)$)$)
					    coordinate (\fixname_in\i)}\fi 
			\ifnum #2=0{} \else foreach \i in {1,...,#2} {
			   ($($(\fixname.north
			   east)+(0,.9\ibbportsep)$)!{\i/(#2+1)}!($(\fixname.south
			   east)-(0,.9\ibbportsep)$)$) coordinate (\fixname_out\i)}\fi;
		   }}}
   },
   blankbb name/.style={append after command={\pgfextra{\node[anchor=north] at
     (\fixname.north) {#1};}}},
   symbb port sep/.store in=\symbbportsep,
   symbb port sep=2,
   symbb port length/.store in=\symbbportlen,
   symbb port length=0pt,
   symbb min width/.store in=\symbbminwidth,
   symbb min width=1cm,
   symbb rounded corners/.store in=\symbbcorners,
   symbb rounded corners=5pt,
   symbb small/.style={symbb port sep=1, symbb port length=2.5pt, symbbx=.4cm, symbb min width=.4cm, symbby=.7ex},
   symbb/.code 2 args={
      \pgfmathsetlengthmacro{\symbbheight}{\symbbportsep * (max(#1,#2)) * \bby}
      \pgfkeysalso{draw,minimum height=\symbbheight,minimum width=\symbbminwidth,outer sep=0pt,
         rounded corners=\symbbcorners,thick,
         prefix after command={\pgfextra{\let\fixname\tikzlastnode}},
         append after command={\pgfextra{\draw
            \ifnum #1=0{} \else foreach \i in {1,...,#1} {
               ($($(\fixname.north
	       west)+(0,.9\symbbportsep)$)!{\i/(#1+1)}!($(\fixname.south
	       west)-(0,.9\symbbportsep)$)$)
	       +(-\symbbportlen,0) coordinate (\fixname_in\i) -- +(\symbbportlen,0) coordinate (\fixname_in\i')}\fi 
            \ifnum #2=0{} \else foreach \i in {1,...,#2} {
               ($($(\fixname.north
	       east)+(0,.9\symbbportsep)$)!{\i/(#2+1)}!($(\fixname.south
	       east)-(0,.9\symbbportsep)$)$) +(-\symbbportlen,0) coordinate (\fixname_out\i') -- +(\symbbportlen,0) coordinate (\fixname_out\i)}\fi;
         }}}
   },
   symbb name/.style={append after command={\pgfextra{\node[anchor=north] at
(\fixname.north) {#1};}}},
}

\newtheorem*{theorem*}{Theorem}
\newtheorem{definition}{Definition}[section]
\newtheorem{proposition}[definition]{Proposition}   
\newtheorem{theorem}[definition]{Theorem}
\newtheorem{lemma}[definition]{Lemma}

\theoremstyle{remark}
\newtheorem{example}[definition]{Example}
\newtheorem*{example*}{Example}
\newtheorem{remark}[definition]{Remark}

\title{Lenses and Learners}
\author{Brendan Fong and Michael Johnson}
\author{
Brendan Fong \\ 
Department of Mathematics \\
Massachusetts Institute of Technology \\ 
bfo@mit.edu
\and
Michael Johnson \\  
Faculty of Science and Engineering\\
Macquarie University \\ 
michael.johnson@mq.edu.au
}

\institution{}

\begin{document}
\maketitle

\begin{abstract}
  Lenses are a well-established structure for modelling bidirectional
  transformations, such as the interactions between a database and a view of
  it. Lenses may be symmetric or asymmetric, and may be composed, forming
  the morphisms of a monoidal category. More recently, the notion of a learner
  has been proposed: these provide a compositional way of modelling supervised
  learning algorithms, and again form the morphisms of a monoidal category. In
  this paper, we show that the two concepts are tightly linked. We show
  both that there is a faithful, identity-on-objects symmetric monoidal functor
  embedding a category of \emph{asymmetric} lenses into the category of
  learners, and furthermore there is such a functor embedding the category of
  learners into a category of \emph{symmetric} lenses.
\end{abstract}
\vskip 32pt

\section{Introduction}

This paper presents surprising links between two apparently disparate areas:
a compositional treatment of supervised learning algorithms, called \emph{learners},
and and a mathematical treatment of bidirectional transformations, called
\emph{lenses}. Until this work there had been no known non-trivial relationships
between the two areas, and, naively at least, there seemed to be little
reason to expect them to be closely related mathematically.

But lenses and learners are indeed mathematically closely related. The main
result that we present here is the existence of a faithful,
identity-on-objects, symmetric monoidal functor, from the category whose arrows are learners
to a category whose arrows are \emph{symmetric lenses}. In addition, the symmetric
lenses in the image of that functor have particularly simple structure: as
spans of asymmetric lenses, their left legs are the long studied and easily
understood lenses known as \emph{constant complement lenses}. Roughly speaking,
this means that a supervised learning algorithm may be understood as a special
type of symmetric lens.

The right legs of the symmetric lenses in the image of the functor are also
simple, but in a different way. The right legs are bare asymmetric lenses ---
lenses which have, as all lenses do, a Put and Get, but which have no axioms
restricting the way the Put and Get interact. Such simple lenses were defined
in the very first papers on lenses, and were the ones blessed with the name
``lens'', but the study of such bare lenses has so far been somewhat
neglected by the lens community, and certainly by us, because of a lack of
examples or applications requiring them. This paper provides some compelling
examples and has led us to study them seriously.

The symmetric lenses that correspond to learners also have links, as will be
explained in the discussion section, with recent work on some quite
sophisticated notions of lenses called \emph{amendment lenses}, or in a slightly
simplified form, \emph{a-lenses}. These a-lenses do have extra axioms --- the ones
we have looked at most are called Stable PutGet (or SPG) a-lenses --- and it
turns out that the extra structure provided by amendment lenses can be added
in a canonical way to the lenses that we study here, and in a way that allows
then to recover at least the amendment-lens versions of GetPut and PutGet
axioms. 

\subsubsection*{Terminology and notation}
Before we begin, a few words about notation. In particular, we may need to thank readers for their indulgence in translating notations --- since this paper
brings together two areas that already have established notation, in order to make the comparisons 
with those areas easier we have generally retained the appropriate notations for lenses and for learners,
but that means that readers sometimes have to carry both notations in their minds and translate
between them.  We have endeavoured to help.

In a small number of cases where we judged it would do little harm we have taken the 
opposite approach and made minor changes to established notation.  In those cases, if the
notation comes from the readers' own area of expertise we have to ask for even more indulgence.  We
hope we do not engender confusion.

On more mathematical matters, note that we are both category theorists, but
that the lenses and learners that we are dealing with here are all set-based,
and there is very little explicit use of category theory beyond some
pullbacks, the construction of functions from other functions, and the
judicious, and mostly unremarked, use of isomorphisms to for example re-order
variables inside tuples, and to rebracket tuples of tuples. But sadly that
makes for some long strings of variables as parameters for functions. 

To manage the visual complexity of some of these complicated function
compositions we have also presented them using string diagrams in $\set$.
While we hope that, written alongside the usual elementwise function
notation, these will be easy enough to read, the reader unfamiliar with this
notation might consult introductory references \cite{Selinger,FS19}. We look
forward to a treatment of some of the things presented here as notions
internal to a category, and anticipate that they will probably bring with
them significant simplifications, especially in notation.

Finally, a remark about how ``well-behaved'' is a technical term.  Many of the lenses we consider
here are not well-behaved, but that simply means that they don't satisfy certain conditions,
conditions which for the applications considered here would be undesirable.  So it's important
to remember that being not-well-behaved may be desirable and is certainly not derogatory (contrary 
to normal English usage).

\subsubsection*{Outline}
The paper is structured as follows.  

In the next two sections we introduce lenses,
firstly in their asymmetric form (Section~\ref{sec-AsymmetricLenses}), and
then in their symmetric form (Section~\ref{sec-SymmetricLenses}). Along the
way we note how to compose both asymmetric lenses, and symmetric lenses, and
we introduce the particularly simple constant complement lenses, and show
briefly how constant complement lenses compose to give constant complement
lenses. Because of the bare lenses that we are using in this paper, we need
to extend somewhat the usual definition of composition of symmetric lenses
which is generally defined only when the lenses are well-behaved, or at least
each satisfy
the PutGet axiom.  We show that, with a slightly complicated 
construction, the usual definition of composition can be extended to symmetric lenses in 
which one leg satisfies PutGet even if the other leg satisfies no axioms at all.

In Section~\ref{sec-Learners} we introduce learners, their composition and monoidal
structure, and how they form a category.  As might be expected, to make a category
(rather than say a bicategory) we need to take equivalence classes of learners, and the
equivalence relation is introduced.  It corresponds so closely to the usual equivalence
for symmetric lenses presented as spans of asymmetric lenses that we can suppress any 
detailed treatment of the equivalences in this paper and present the results in terms of 
representatives of equivalence classes.  Some readers at first find the definition of 
composition for learners a little daunting, so we include string diagrams to illustrate 
how the composition works.

In Section~\ref{sec-MainResult} we present in detail the theorem already alluded
to, and illustrate the precise and remarkably parallel relationship between
learners and lenses.  And then in Section~\ref{sec-Discussion} we discuss some
of the observations that follow from this work, and conclude with some
speculations on possible directions for further studies flowing from these
results. 

\section{Asymmetric Lenses} \label{sec-AsymmetricLenses}

Asymmetric lenses have been studied in a variety of different categories, and with a range of 
different forms including d-lenses \cite{ddl}, c-lenses \cite{jrdl}, and most recently 
amendment lenses, also known as a-lenses \cite{dml}.  Furthermore those lenses have been 
studied as algebraic structures with a number of axioms such as the so-called
PutGet law \cite{fgmps}, the PutPut law \cite{fgmps,jrlpp}, and many others.  In this paper we predominantly
restrict our attention to the very simplest cases  of set based lenses, 
as originally presented in \cite{pslvut}, with no further axioms.

\begin{definition}
  An \define{asymmetric lens} $(p,g)\colon A \to B$ is a pair of functions: a
  \define{Put} $p\colon B \times A \to A$ and a \define{Get} $g\colon A
  \to B$.
\end{definition}

These basic asymmetric lenses are the main thing that we need in order to see the relationships with 
learners, but we will sometimes need to refer to lenses that are ``well-behaved'', or at least that 
satisfy one of the following two conditions required for so-called well-behaved lenses.

\begin{definition}
  An asymmetric lens $(p,g)\colon A \to B$ is
  called \define{well-behaved} if it satisfies the following two conditions 
  \begin{description}
    \item[(PutGet)]  The Get of a Put is the projection. That is, $g(p(a,b)) = b$, or in string diagrams
          
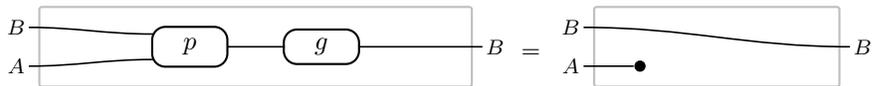
\[
            \begin{aligned}
              \begin{tikzpicture}[oriented WD]
                \node[bb port sep=1, bb={2}{1}]                            (I)     {$p$};
                \node[bb port sep=1, bb={1}{1}, right=.5 of I](J)     {$g$};
                \node[ibb={2}{1}, fit=(I) (J)]                          (outer) {};
                \node at ($(outer_in1')-(0.3,0)$) {\footnotesize $B$};
                \node at ($(outer_in2')-(0.3,0)$) {\footnotesize $A$};
                \node at ($(outer_out1')+(0.3,0)$) {\footnotesize $B$};
                \draw (outer_in1) to (I_in1);
                \draw (outer_in2) to (I_in2);
                \draw (I_out1) to (J_in1);
                \draw (J_out1) to (outer_out1);
              \end{tikzpicture}
            \end{aligned}
            =
            \begin{aligned}
              \begin{tikzpicture}[oriented WD]
                \node                           (I)     {};
                \node[ibb={2}{1}, fit=(I)]                          (outer) {};
                \node[circle, inner sep=1.5, fill] (n) at ($(outer_in2)+(.5,0)$) {};
                \node at ($(outer_in1')-(0.3,0)$) {\footnotesize $B$};
                \node at ($(outer_in2')-(0.3,0)$) {\footnotesize $A$};
                \node at ($(outer_out1')+(0.3,0)$) {\footnotesize $B$};
                \draw (outer_in2) to (n);
                \draw (outer_in1) to (outer_out1);
              \end{tikzpicture}
            \end{aligned}
          \]

    \item[(GetPut)]  The Put of an unchanged Get result is unchanged. That is, $p(g(a),a) = a$, or in string diagrams
          
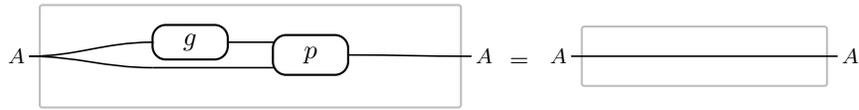
\[
            \begin{aligned}
              \begin{tikzpicture}[oriented WD]
                \node[bb port sep=1, bb={2}{1}] (J)     {$p$};
                \node[bb port sep=1, bb={1}{1}, left=.4 of J_in1]  (I)     {$g$};
                \coordinate (n) at (I_in1|-J_in2);
                \coordinate (p) at ($(n)-(0,1)$);
                \node[ibb={1}{1}, fit=(J) (I) (p)]           (outer) {};
                \node at ($(outer_in1')-(0.3,0)$) {\footnotesize $A$};
                \node at ($(outer_out1')+(0.3,0)$) {\footnotesize $A$};
                \draw (outer_in1) to (I_in1);
                \draw (outer_in1) to (n);
                \draw (n) to (J_in2);
                \draw (I_out1) to (J_in1);
                \draw (J_out1) to (outer_out1);
              \end{tikzpicture}
            \end{aligned}
            =
            \begin{aligned}
              \begin{tikzpicture}[oriented WD]
                \node                           (I)     {};
                \node[ibb={1}{1}, fit=(I)]                          (outer) {};
                \node at ($(outer_in1')-(0.3,0)$) {\footnotesize $A$};
                \node at ($(outer_out1')+(0.3,0)$) {\footnotesize $A$};
                \draw (outer_in1) to (outer_out1);
              \end{tikzpicture}
            \end{aligned}
          \]
  (Here the splitting wire represents the diagonal map $A \to A \times
  A$, i.e.\ $a \mapsto (a,a)$.)

  \end{description}
\end{definition}

Even without meeting the definition of well-behaved, lenses compose in a straightforward
way, and well-behaved lenses do compose to give well-behaved lenses too.  In fact, each
of the two well-behaved conditions (PutGet and GetPut) is respected by composition separately, 
so we can, when we need to, talk about the composition of lenses that satisfy merely PutGet say, 
and know that the result will also satisfy PutGet.

\begin{definition}\label{def-AsymLensComposition}
  The \define{composite lens} $(p,g) \cp (p',g')\colon A \to C$ constructed from lenses
  $(p,g)\colon A \to B$ and $(p',g')\colon B \to C$
  has as Get simply the composite $g'g$ of the Gets
  
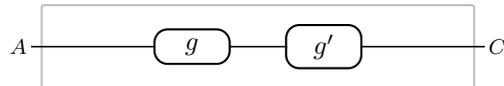
\[
    \begin{aligned}
      \begin{tikzpicture}[oriented WD]
        \node[bb port sep=1, bb={1}{1}](G1)     {$g$};
        \node[bb port sep=1, bb={1}{1}, right=.5 of G1](G2)     {$g'$};
        \node[ibb={1}{1}, fit=(G1) (G2)]                          (outer) {};
        \node at ($(outer_in1')-(0.3,0)$) {\footnotesize $A$};
        \node at ($(outer_out1')+(0.3,0)$) {\footnotesize $C$};
        \draw (outer_in1) to (G1_in1);
        \draw (G1_out1) to (G2_in1);
        \draw (G2_out1) to (outer_out1);
      \end{tikzpicture}
    \end{aligned}
  \]
  and as Put, $q\colon C \times A \to A$ given by $q(c,a) = p(p'(g(a),c),a)$.
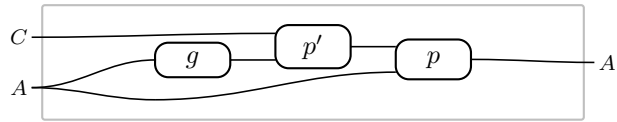
  \[
    \begin{aligned}
      \begin{tikzpicture}[oriented WD]
        \node[bb port sep=1, bb={2}{1}] (P1)     {$p$};
        \node[bb port sep=1, bb={2}{1}, left=.4 of P1_in1] (P2)  {$p'$};
        \node[bb port sep=1, bb={1}{1}, left=.4 of P2_in2]  (G)  {$g$};
        \coordinate (n) at ($(G_in1)-(0,2)$);
        \node[ibb={2}{1}, fit=(P1) (G) (P2) (n)]           (outer) {};
        \node at ($(outer_in1')-(0.3,0)$) {\footnotesize $C$};
        \node at ($(outer_in2')-(0.3,0)$) {\footnotesize $A$};
        \node at ($(outer_out1')+(0.3,0)$) {\footnotesize $A$};
        \draw (outer_in1) to (P2_in1);
        \draw (outer_in2) to (G_in1);
        \draw (G_out1) to (P2_in2);
        \draw (outer_in2) to (n);
        \draw (n) to (P1_in2);
        \draw (P2_out1) to (P1_in1);
        \draw (P1_out1) to (outer_out1);
      \end{tikzpicture}
    \end{aligned}
  \]
  Thus in the notation just introduced, $(p,g) \cp (p',g') = (q,g'g)$.
\end{definition}

\begin{definition}
With the composition just defined, asymmetric lenses are the arrows of a 
symmetric monoidal category whose objects are sets, and whose 
monoidal product is given by cartesian product.  That category is denoted $\alens$.
Furthermore, there is a subcategory of $\alens$ called $\wbalens$ whose arrows
are well-behaved asymmetric lenses, as well as subcategories whose arrows satisfy
merely PutGet or GetPut.
\end{definition}

Given a cartesian product $B \times A$, we write the projection notation
$\pi_2\colon B \times A \to A$ for the function mapping $(b,a)$ to $a$; that
is, for projection onto the second factor. More generally, we overload the notation
$\pi_i$ using it for any projection onto an $i$th factor, with it being
disambiguated once the domain, expressed as a cartesian product, is known.

The identity lens $(\pi_1,\id_A)\colon A \to A$ on a set $A$ has identity
function as Get and projection onto the first factor as Put.

\vskip 0.5cm  
\begin{example}\label{eg-ConstComp}
One of the most basic forms of asymmetric lenses, introduced many years ago
in the database community \cite{bsusrv}, is the \emph{constant complement view updating lens}.
These are asymmetric lenses of the form $(k,\pi_1) \colon A_1 \times A_2 \to A_1$, where
$k(a_1',(a_1,a_2) ) = (a_1',a_2)$ --- the reader can see the source of the
name ``constant complement'' in the presence of $a_2$ in both the input and
the output of $k$. It is easy to see that $(k,\pi_1)$ is well-behaved.

The composite of constant complement lenses is again constant
complement. Indeed, suppose further that $A_1 = B_1 \times B_2$, and that
$(k', \pi_1)\colon B_1 \times B_2 \to B_1$ is a constant complement lens 
(so $k'(b_1',(b_1,b_2)) = (b_1',b_2)$).  Then the composite lens 
$(k, \pi_1) \cp (k', \pi_1) \colon B_1 \times B_2 \times A_2 \to B_1$ has Put 
$k''\colon  B_1 \times (B_1 \times B_2 \times A_2) \to B_1 \times B_2 \times A_2$ given by
\begin{align*}
k''\big(b'_1,(b_1,b_2,a_2)\big) 
&= k\big(k'(b'_1,(b_1,b_2)),(b_1,b_2,a_2)\big) \\
&= k\big((b_1',b_2),(b_1,b_2,a_2)\big) \\
&= (b'_1,b_2,a_2),
\end{align*}
and Get given by $\pi_1\colon B_1 \times (B_2 \times A_2) \to B_1$.
This is, up to isomorphism, the constant complement lens 
$(k'',\pi_1) \colon B_1 \times B_2 \times A_2 \to B_1$.
\end{example}
\vskip 0.5cm

Such simple composites of constant complement lenses, or indeed of other more
complicated lenses, arise very frequently in practice (for example in the
database world whenever one deals with views of views). In common with the
overloaded notation for projections $\pi_i$ which are the Gets of constant
complement lenses, it is convenient to introduce overloaded notation for the
Puts: When there is little risk of confusion we will simply write $k$ for the
constant complement Put corresponding to a given Get $\pi_i$.

We will return to constant complement lenses when we use them as the ``left leg'' 
of certain symmetric lenses to obtain our main result in Section~\ref{sec-MainResult}.

\section{Symmetric Lenses}\label{sec-SymmetricLenses}

Asymmetric lenses model well situations where one system, denoted $A$ in the
previous section, includes all the relevant information, and another system,
$B$, has information simply derived from $A$. Of course in practice it's
important to deal with the more symmetric situation where two systems share
some common structures, but each has information that the other doesn't have.
To address this, not long after the seminal work on asymmetric lenses,
Hofmann \emph{et al.} developed a symmetrized version of lenses \cite{hpwsl}
which can be described conveniently as (approximately) spans of asymmetric
lenses.

Recall that a \define{span} in a category is a pair of arrows with common
domain, $A_1 \longleftarrow S \longrightarrow A_2$. Despite their evident
symmetry, spans are usually considered to be oriented from one of the
\define{feet} to the other, so the span just drawn would be called a span
from $A_1$ to $A_2$, and the same two arrows also form a span from $A_2$ to
$A_1$, usually drawn as $A_2 \longleftarrow S \longrightarrow A_1$. The
object $S$ is called the \define{head}, (or sometimes \define{peak} or
\define{apex}) of the span. When we need to name the arrows, for example if
they are asymmetric lenses $(p_1,g_1)\colon S \longrightarrow A_1$ and
$(p_2,g_2)\colon S \longrightarrow A_2$, it is convenient to notate the span
as $ A_1 \xleftarrow{(p_1,g_1)} S \xrightarrow{(p_2,g_2)} A_2$. The lens
$(p_1,g_1)$ will be called the \define{left leg} of the span, and the lens
$(p_2,g_2)$ will similarly be called the \define{right leg} of the span.

The reader may choose to think of a symmetric lens, shortly to be introduced, as a
span of asymmetric lenses.  But in fact, in common with other descriptions of
similar structures, symmetric lenses are really equivalence classes of spans of
asymmetric lenses.  In this paper it will be convenient to talk about, for example,
``the symmetric lens $A_1 \xleftarrow{(p_1,g_1)} S \xrightarrow{(p_2,g_2)} A_2$'',
when strictly speaking that span is just a representative of an equivalence class
of similar spans, and that equivalence class is the symmetric lens.

For completeness we include here the description of the equivalence relation.

\begin{definition} \label{def-equivslens}
Suppose given two spans of asymmetric lenses
$A_1 \xleftarrow{(p_1,g_1)} S \xrightarrow{(p_2,g_2)} A_2$ and
$A_1 \xleftarrow{(p'_1,g'_1)} S' \xrightarrow{(p'_2,g'_2)} A_2$
with common feet $A_1$ and $A_2$. A function $\equivf\colon S \longrightarrow
S'$ is said to satisfy \define{conditions (E)} \cite{jrjot} when
\begin{enumerate}[label=(\roman*)]
\item $\equivf$ is surjective.
\item $\equivf$ preserves Gets: $g'_1 \equivf = g_1$ and $g'_2 \equivf = g_2$.
\item $\equivf$ preserves Puts: for all $a_1$, $s$ we have
  $p'_1(a_1,\equivf(s)) = \equivf(p_1(a_1,s))$ and $p'_2(a_2,\equivf(s)) =
  \equivf(p_2(a_2,s))$.
\end{enumerate}
\end{definition}
 
Let $\equiv_{sp}$ be the equivalence relation on spans of asymmetric lenses
generated by those functions $\equivf$ between their heads that satisfy conditions (E).

\begin{definition}
A \define{symmetric lens} from $A_1$ to $A_2$ is a $\equiv_{sp}$-class of 
spans of asymmetric lenses from $A_1$ to $A_2$.
\end{definition}

The composition of symmetric lenses is usually defined for (equivalence
classes of) spans of well-behaved asymmetric lenses \cite{jrjot}, but we aim to be
more general here. In particular, we will show that the composition rule for
well-behaved symmetric lenses generalises to a composition rule for
\emph{symmetric lenses with left legs satisfying PutGet}.

Note first that if a symmetric lens has as a representative a span of asymmetric
lenses in which one leg satisfies PutGet then all the representatives of that equivalence
class have corresponding leg satisfying PutGet.

\begin{definition}\label{def-SymmetricLensComposition}
Suppose that  
\[
  A_1 \xleftarrow{(q_1,h_1)} S_1 \xrightarrow{(p_2,g_2)} A_2 \qquad \mbox{and} \qquad
  A_2 \xleftarrow{(q_2,h_2)} S_2 \xrightarrow{(p_3,g_3)} A_3
\]
are spans of asymmetric lenses whose left legs satisfy PutGet. We define
their \define{composite symmetric lens}, from $A_1$ to $A_3$, as follows.

Let $S_1 \xleftarrow{\overline{h_2}} T \xrightarrow{\overline{g_2}} S_2$ be
the pullback in $\set$ of the cospan $S_1 \xrightarrow{g_2} A_2
\xleftarrow{h_2} S_2$. More concretely, without loss of generality, we may
suppose that 
\[
  T = \{(s_1,s_2) \in S_1 \times S_2 \,\mid\, g_2(s_1) = h_2(s_2)\}, 
\] 
and that $\overline{h_2}$ and $\overline{g_2}$ are the
projections onto the first and second components respectively. 
Next, we equip $\overline{h_2}$ with the Put $\overline{q_2}\colon S_1 \times T \longrightarrow T$ defined by $\overline{q_2}\big(s'_1,(s_1,s_2)\big) =
\big(s'_1,\, q_2(g_2(s'_1),s_2)\big)$, and equip $\overline{g_2}$ with the
Put $\overline{p_2}\colon S_2 \times T \longrightarrow T$ defined by
\[
  \overline{p_2}\big(s'_2,(s_1,s_2)\big) = \big(p_2(h_2(s'_2),s_1),\,
q_2(g_2(p_2(h_2(s'_2),s_1)),s'_2)\big)  \in T \subseteq S_1\times S_2.
\] 
A representative for the composite symmetric lens $A_1 \to A_3$ is then given by
\[
A_1 \xleftarrow{(\overline{q_2},\overline{h_2}) \cp (q_1,h_1)} T 
\xrightarrow{(\overline{p_2},\overline{g_2}) \cp (p_3,g_3)} A_2.
\]
\end{definition}

To help parse these expressions, we draw string diagrams for $\overline{q_2}$ 
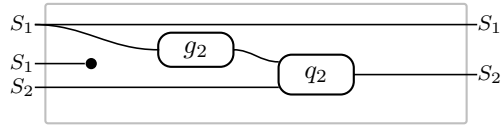
\[
  \begin{aligned}
    \begin{tikzpicture}[oriented WD]
      \node[bb port sep=1, bb={2}{1}] (Q)     {$q_2$};
      \node[bb port sep=1, bb={1}{1}, above left=-.25 and .4 of Q_in1]  (G)     {$g_2$};
      \node[ibb={3}{2}, fit=(G) (Q)]           (outer) {};
      \node[circle, inner sep=1.5, fill] (n) at ($(outer_in2)+(.5,0)$) {};
      \node at ($(outer_in1')-(0.3,0)$) {\footnotesize $S_1$};
      \node at ($(outer_in2')-(0.3,0)$) {\footnotesize $S_1$};
      \node at ($(outer_in3'|-Q_in2)-(0.3,0)$) {\footnotesize $S_2$};
      \node at ($(outer_out1'|-outer_in1)+(0.3,0)$) {\footnotesize $S_1$};
      \node at ($(outer_out2'|-Q_out1)+(0.3,0)$) {\footnotesize $S_2$};
      \draw (outer_in1) to (outer_out1|-outer_in1);
      \draw (outer_in1) to (G_in1);
      \draw (outer_in2) to (n);
      \draw (outer_in3|-Q_in2) to (Q_in2);
      \draw (G_out1) to (Q_in1);
      \draw (Q_out1) to (outer_out2|-Q_out1);
    \end{tikzpicture}
  \end{aligned}
\]
and $\overline{p_2}$
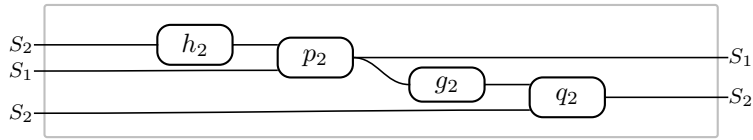
\[
  \begin{aligned}
    \begin{tikzpicture}[oriented WD]
      \node[bb port sep=1, bb={2}{1}] (Q)     {$q_2$};
      \node[bb port sep=1, bb={1}{1}, left=.4 of Q_in1]  (G)     {$g_2$};
      \node[bb port sep=1, bb={2}{1}, above left=-.5 and .5 of G] (P)     {$p_2$};
      \node[bb port sep=1, bb={1}{1}, left=.4 of P_in1]  (H)     {$h_2$};
      \node[ibb={3}{2}, fit=(G) (Q) (P) (H)]           (outer) {};
      \node at ($(outer_in1'|-H_in1)-(0.3,0)$) {\footnotesize $S_2$};
      \node at ($(outer_in2')-(0.3,0)$) {\footnotesize $S_1$};
      \node at ($(outer_in3')-(0.3,0)$) {\footnotesize $S_2$};
      \node at ($(outer_out1'|-P_out1)+(0.3,0)$) {\footnotesize $S_1$};
      \node at ($(outer_out2'|-Q_out1)+(0.3,0)$) {\footnotesize $S_2$};
      \draw (outer_in1|-H_in1) to (H_in1);
      \draw (outer_in2) to (P_in2);
      \draw (outer_in3) to (Q_in2);
      \draw (H_out1) to (P_in1);
      \draw (P_out1) to (outer_out1|-P_out1);
      \draw (P_out1) to (G_in1);
      \draw (G_out1) to (Q_in1);
      \draw (Q_out1) to (outer_out2|-Q_out1);
    \end{tikzpicture}
  \end{aligned}
\]
It is straightforward to check that the images of the functions
$\overline{q_2}$ and $\overline{p_2}$ do indeed lie in their codomain $T$,
and hence that they are well-defined. Note that the construction above gives a
diagram of asymmetric lenses
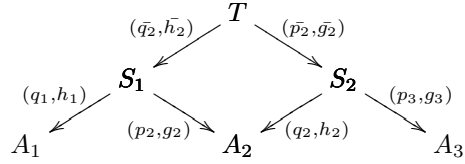
\[
\bfig
\Atriangle(400,250)/->`->`/<400,250>[T`S_1`S_2;(\bar{q_2},\bar{h_2})`(\bar{p_2},\bar{g_2})` ]
\Vtriangle(400,0)/`->`->/<400,250>[S_1`S_2`A_2; `(p_2,g_2)`(q_2,h_2)]
\morphism(400,250)|a|<-400,-250>[S_1`A_1;(q_1,h_1)]
\morphism(1200,250)|a|<400,-250>[S_2`A_3;(p_3,g_3)]
\efig
\]
The composite symmetric lens is just given by composing the pairs of arrows
on the left and right of the diagram. It is straightforward to verify that $(\overline{q_2},\overline{h_2})$
satisfies PutGet, and hence that the left leg of the composite also satisfies
PutGet.

\begin{remark}
  The somewhat complicated second component $q_2(g_2(p_2(h_2(s'_2),s_1),s'_2))$ of $\overline{p_2}$ arises
  because we do not assume that $(p_2,g_2)$ satisfies PutGet. If we do assume
  $(p_2,g_2)$ obeys PutGet, then the expression simplifies to the more familiar
  $q_2(h_2(s'_2),s'_2)$. This second component correspondingly also shows
  that $(\overline{p_2},\overline{g_2})$ may also fail to satisfy PutGet.
\end{remark}



\begin{remark}
It is worth remarking, as noted elsewhere \cite{jrjot}, that the composition just 
defined is reminiscent of, but different from, the normal composition of spans in
a category with pullbacks.  The peak of the span, $T$, is indeed a pullback, but not
in the category $\alens$.  Since the pullback is calculated in $\set$, the pullback
projections are not a priori lenses, but the construction shows how to extend them 
to be lenses in a canonical way.

The construction presented here is also more general than the usual construction
because it has to deal with (right leg) lenses that might not satisfy PutGet.  We will comment 
further on this in the discussion section below. 
\end{remark}

Using the techniques of \cite{jrjot}, it can be shown that $\equiv_{sp}$ is a
congruence for this more general composition of spans of asymmetric lenses.
Thus the composition is well-defined on equivalence classes and provides a 
symmetric lens composition for those symmetric lenses whose left leg satisfies
PutGet.  We thus make the following definition.

\begin{definition}
  We define the symmetric monoidal category $\slens$ to have sets as objects,
  symmetric lenses with left leg satisfying PutGet as arrows, and monoidal
  product given by cartesian product of sets.
\end{definition}

Again, the reader should note that this is slightly more general than other 
definitions of categories of symmetric lenses, which normally require both legs
to satisfy PutGet (and possibly other conditions too). Nonetheless, using the
standard techniques it is straightforward to check that this composition rule
is associative and unital, and moreover that $\slens$ is a well-defined
symmetric monoidal category. Identity symmetric lenses are simply given by the
span in which both legs are identity asymmetric lenses.

\section{Learners}\label{sec-Learners}

Learners provide a categorical framework for modelling supervised learning
algorithms. They can be seen as parametrised version of asymmetric lenses. In
this section we introduce the basic ideas; more detail can be found in
\cite{FST17}.

\begin{definition}
  A \define{learner} $(P,I,U,r)\colon A \to B$ is a set $P$, together with
  three functions: an \define{implementation} $I\colon P \times A \to B$, an
  \define{update} $U\colon B \times P\times A \to P$, and a \define{request}
  $r\colon B \times P \times A \to A$.
\end{definition}

The goal of supervised learning is to approximate a function $f\colon A \to B$
using pairs $(a,f(a)) \in A \times B$ of sample values, or \emph{training data}.
We view $P$ as a set of parameters, and the implementation function as detailing how
this set $P$ parametrises functions, seen as hypotheses, $A \to B$. Next, given
a current hypothesis $p \in P$ and training datum $(a,b) \in A \times B$, the
update and request functions describe two ways to react to differences between
$I(p,a)$ and $b$: first by \emph{updating} the hypothesis $p$ to $U(b,p,a)$, and
second by \emph{requesting} an alternative input $r(b,p,a)$.

\begin{remark}
  While the implementation and update functions are evidently necessary
  structure for supervised learning, the role of the request function is more
  subtle. Indeed, the request function only becomes necessary through
  compositional considerations: it is what permits the construction of new
  learners by interconnecting given ones. Crucially, it captures the
  backpropagation part of the widely-used backpropagation algorithm for
  efficient training neural networks. Further discussion regarding
  interpretation of the request function can be found in \cite[Remark
  II.2]{FST17}.
\end{remark}

\begin{example} \label{eg-EuclideanLearners}
  Learners are not required to obey any axioms, and so are straightforward to
  construct. There are, however, learners which have been shown to be more useful than others in
  practice. One useful way of constructing learners is by using gradient
  descent on any differentiably parametrised class of functions, such as one
  defined using a neural net.

  Indeed, given a set $\rr^k$ and differentiable function $I\colon \rr^k \times
  \rr^m \to \rr^n$, as well as a real number $\epsilon > 0$ that we call a
  \emph{step size}, we may define a learner $(P,I,U,r) \colon \rr^m \to \rr^n$
  by setting
  \[
    U(b,p,a) = p - \epsilon \nabla_p \tfrac12 \lVert I(p,a) - b \rVert^2,
  \]
  \[
    r(b,p,a) = a - \nabla_a \tfrac12 \lVert I(p,a) - b \rVert^2,
  \]
  where $\lVert x \rVert$ is the Euclidean norm on $\rr^n$.

  A key property of the category of learners is that this interpretation of a
  differentiable function $I$ is functorial, and indeed this functor captures
  the structure of the backpropagation algorithm. For more details see
  \cite[Theorem III.2]{FST17}.
\end{example}

To state the aforementioned functoriality result, we must first describe what
it means to compose learners. As with symmetric lenses, this first relies on
stating what it means for two learners to be equivalent.

\begin{definition} \label{def-equivlearn}
Given two learners $(P,I,U,r), (P',I',U',r')\colon A \to B$ a function
$\equivf\colon P \to P'$ is said to satisfy \define{conditions $(E')$} when
\begin{enumerate}[label=(\roman*)]
\item $\equivf$ is surjective.
\item $\equivf$ preserves implementations: $I'(\equivf(p),a) = I(p,a)$.
\item $\equivf$ preserves updates: $U'(b,\equivf(p),a) = \equivf(U(b,p,a))$.
\item $\equivf$ preserves requests: $r'(b,\equivf(p),a) = r(b,p,a)$.
\end{enumerate}
\end{definition}
Just as for symmetric lenses, this generates an equivalence relation $\equiv_l$ on learners.
Equivalence classes of learners form the morphisms of a symmetric monoidal category.




\begin{definition} \label{def-compositelearners}
  The symmetric monoidal category $\lrn$ has sets as objects and $\equiv_l$-classes of learners as morphisms.

  The composite of learners
  \[
    A \xrightarrow{(P,I,U,r)} B \xrightarrow{(Q,J,V,s)} C.
  \]
  is defined to be $(Q\times P,\, I \ast J,\, U \ast V,\, r\ast s)$, where
  \begin{align*}
    (I \ast J)(q,p,a) &= J(q,I(p,a)), \\
    (U\ast V)(c,q,p,a) &= \Big(U\big(s(c,q,I(p,a)),p,a\big),
    V\big(c,q,I(p,a)\big)\Big), \\
    (r\ast s)(c,q,p,a) &= r\big(s(c,q,I(p,a))\big).
  \end{align*}

  The monoidal product of objects $A$ and $B$ is their cartesian product $A \times
  B$, while the monoidal product of morphisms $(P,I,U,r)\colon A \to B$ and
  $(Q,J,V,s)\colon C \to D$ is $(P\times Q,\,I\prl J,\,U\prl V,\,r\prl s)$, where
  the implementation function is
  \begin{align*}
    (I\prl J)(p,q,a,c) &= (I(p,a),J(q,c)), \\
    (U\prl V)(b,d,p,q,a,c) &= (U(b,p,a),V(d,q,c)), \\
    (r\prl s)(b,d,p,q,a,c) &= (r(b,p,a),s(d,q,c)).
  \end{align*}
\end{definition}

\begin{remark}
A proof that this definition indeed specifies a well defined symmetric monoidal
category follows from the same arguments as those given in \cite[Proposition
II.4]{FST17}.  Note, however, a key change: in the setting of \cite{FST17},
conditions (E$'$) are strengthened to require $f$ be a bijection. The
requirement that $f$ be a bijection was made to avoid a digression about
differentiability in \cite[Definition III.1]{FST17}, and yet still permit a
straightforward statement of the main theorem \cite[Theorem III.2]{FST17}.
Nonetheless, the authors of \cite{FST17} believe conditions (E$'$) give the more
natural notion of equivalence of learner, as it allows identification of
parameters that have the same implementation. We believe that the correspondence
with conditions (E) from \cite{jrjot} via Theorem~\ref{thm-main} provides further
evidence of this claim; indeed, we view this added clarity as a positive
outcome of this work and the interaction between our two communities.
\end{remark}

For clarity, let us also present the composition rule using string diagrams in
$(\mathsf{Set},\times)$. Given learners $(P,I,U,r)$ and $(Q,J,V,s)$ as above, the composite
implementation function can be written as
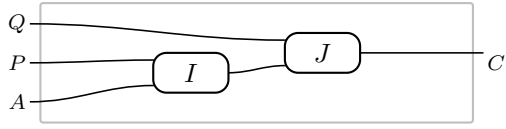
\[
\begin{tikzpicture}[oriented WD]
	\node[bb port sep=1, bb={2}{1}]                            (I)     {$I$};
	\node[bb port sep=1, bb={2}{1}, above right=-1 and .5 of I](J)     {$J$};
	\node[ibb={3}{1}, fit=(I) (J)]                          (outer) {};
	\node at ($(outer_in1')-(0.3,0)$) {\footnotesize $Q$};
	\node at ($(outer_in2')-(0.3,0)$) {\footnotesize $P$};
	\node at ($(outer_in3')-(0.3,0)$) {\footnotesize $A$};
	\node at ($(outer_out1')+(0.3,0)$) {\footnotesize $C$};
	\draw (outer_in1) to (J_in1);
	\draw (outer_in2) to (I_in1);
	\draw (outer_in3) to (I_in2);
	\draw (I_out1') to (J_in2);
	\draw (J_out1') to (outer_out1|-J_out1);
\end{tikzpicture}
\]
while the composite update--request function $(U \ast V, r \ast s)$ can be
written as:
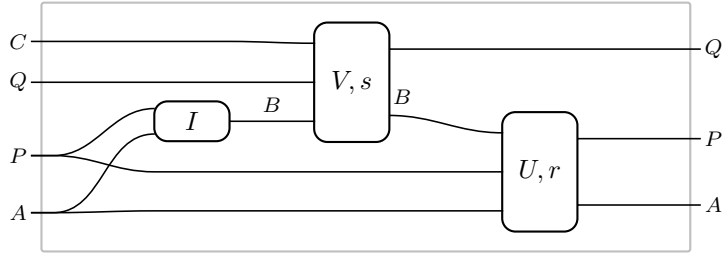
\[
\begin{tikzpicture}[oriented WD]
	\node[bb port sep=2, bb={3}{2}](V)   {$V,s$};
	\node[bb port sep=1, bb={2}{1}, left=.75 of V_in3]                 (I)    {$I$};
	\node[bb port sep=2, bb={3}{2}, below right=-1.5 and 1 of V](U)	{$U,r$};
	\node[ibb={4}{3}, fit=(I) (V) (U)]                          (outer) {};
	\begin{scope}[font=\footnotesize]
  	\node at ($(outer_in1')-(0.3,0)$) {$C$};
  	\node at ($(outer_in2'|-V_in2)-(0.3,0)$) {$Q$};
  	\node at ($(outer_in3')-(0.3,0)$) {$P$};
  	\node at ($(outer_in4')-(0.3,0)$) {$A$};
  	\node at ($(outer_out1'|-V_out1)+(0.3,0)$) {$Q$};
  	\node at ($(outer_out2'|-U_out1)+(0.3,0)$) {$P$};
  	\node at ($(outer_out3'|-U_out2)+(0.3,0)$) {$A$};
  	\draw let \p1=(I.south east), \p2=($(outer_in1)$), \n1=\bbportlen in
  		(outer_in1) to (\x1+\n1, \y2) to (V_in1);
  	\draw (outer_in2|-V_in2) to (V_in2);
  	\draw (outer_in3) -- ($(outer_in3)+(.2,0)$) to (I_in1);
  	\draw let \p1=(I.south west), \p2=($(U_in2)$), \n1=\bbportlen in
  		(outer_in3) -- ($(outer_in3)+(.2,0)$) to (\x1-\n1, \y2) -- (U_in2);
  	\draw (outer_in4) -- ($(outer_in4)+(.2,0)$) to (I_in2);
  	\draw let \p1=(I.south west), \p2=($(U_in3)$), \n1=\bbportlen in
  		(outer_in4) -- ($(outer_in4)+(.2,0)$)to (\x1-\n1, \y2) -- (U_in3);
  	\draw (V_out1) to (outer_out1|-V_out1);
  	\draw (V_out2) to node[above=1pt, pos=.1] {$B$} (U_in1);
  	\draw (U_out1) to (outer_out2|-U_out1);
  	\draw (U_out2) to (outer_out3|-U_out2);
  	\draw (I_out1) to node[above] {$B$} (V_in3);
	\end{scope}
\end{tikzpicture}
\]

The monoidal product of learners is represented in string diagrams as follows.
The product implementation function $I \prl J$ is
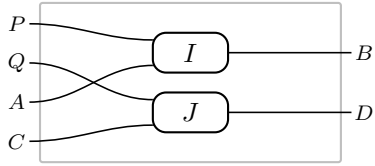
\[
\begin{tikzpicture}[oriented WD]
	\node[bb port sep=1, bb={2}{1}]                            (I)     {$I$};
	\node[bb port sep=1, bb={2}{1}, below= of I]                  (J)     {$J$};
	\node[ibb={4}{2}, fit=(I) (J)]                          (outer) {};
	\node at ($(outer_in1')-(0.3,0)$) {\footnotesize $P$};
	\node at ($(outer_in2')-(0.3,0)$) {\footnotesize $Q$};
	\node at ($(outer_in3')-(0.3,0)$) {\footnotesize $A$};
	\node at ($(outer_in4')-(0.3,0)$) {\footnotesize $C$};
	\node at ($(outer_out1'|-I_out1)+(0.3,0)$) {\footnotesize $B$};
	\node at ($(outer_out2'|-J_out1)+(0.3,0)$) {\footnotesize $D$};
	\draw (outer_in1) to (I_in1);
	\draw (outer_in2) to (J_in1);
	\draw (outer_in3) to (I_in2);
	\draw (outer_in4) to (J_in2);
	\draw (I_out1) to (outer_out1|-I_out1);
	\draw (J_out1) to (outer_out2|-J_out1);
\end{tikzpicture}
\]
while the composite update and request function $(U \prl V, r\prl s)$ is
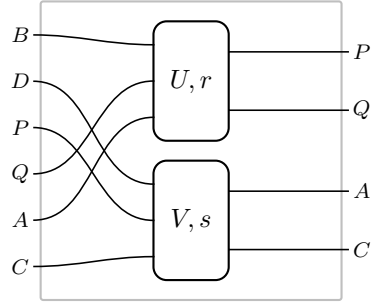
\[
\begin{tikzpicture}[oriented WD,scale=.7]
	\node[bb port sep=2, bb={3}{2}]                            (U)     {$U,r$};
	\node[bb port sep=2, bb={3}{2}, below= of U]               (V)     {$V,s$};
	\node[ibb={6}{4}, fit=(U) (V)]                          (outer) {};
	\node at ($(outer_in1')-(0.35,0)$) {\footnotesize $B$};
	\node at ($(outer_in2')-(0.35,0)$) {\footnotesize $D$};
	\node at ($(outer_in3')-(0.35,0)$) {\footnotesize $P$};
	\node at ($(outer_in4')-(0.35,0)$) {\footnotesize $Q$};
	\node at ($(outer_in5')-(0.35,0)$) {\footnotesize $A$};
	\node at ($(outer_in6')-(0.35,0)$) {\footnotesize $C$};
	\node at ($(outer_out1'|-U_out1)+(0.35,0)$) {\footnotesize $P$};
	\node at ($(outer_out2'|-U_out2)+(0.35,0)$) {\footnotesize $Q$};
	\node at ($(outer_out3'|-V_out1)+(0.35,0)$) {\footnotesize $A$};
	\node at ($(outer_out4'|-V_out2)+(0.35,0)$) {\footnotesize $C$};
	\draw (outer_in1) to (U_in1);
	\draw (outer_in2) to (V_in1);
	\draw (outer_in3) to (V_in2);
	\draw (outer_in4) to (U_in2);
	\draw (outer_in5) to (U_in3);
	\draw (outer_in6) to (V_in3);
	\draw (U_out1) to (outer_out1|-U_out1);
	\draw (U_out2) to (outer_out2|-U_out2);
	\draw (V_out1) to (outer_out3|-V_out1);
	\draw (V_out2) to (outer_out4|-V_out2);
\end{tikzpicture}
\]

\begin{remark} \label{rem-triviallearner}
Note that lenses are learners with trivial, that is singleton, parameter set (as observed already
by Fong et al \cite{FST17}):
\begin{center}
  \begin{tabular}{c|c}
    \textbf{Learner} $A \to B$ & \textbf{Asymmetric lens} $A \to B$ \\\hline
    Hypotheses $P$ & {\bf $1$}\\
    Implementation $I\colon P \times A \to B$ & Get $g\colon A \to B$ \\
    Update $U\colon B \times P \times A \to P$ & --- \\
    Request $r\colon B \times P \times A \to A$ & Put $p\colon B \times A \to
    A$
  \end{tabular}
\end{center}
\end{remark}

This in fact extends to an inclusion of categories.

\begin{proposition}
  There is a faithful, identity-on-objects, symmetric monoidal functor $\alens \to
  \lrn$.
\end{proposition}
\begin{proof}
  It is straightforward to check that the correspondence laid out above
  preserves composition and monoidal products.
\end{proof}

\section{The Main Result}\label{sec-MainResult}

While it is interesting, it is perhaps not surprising, and may not be especially enlightening, to find 
that lenses are learners with trivial parameter set (which amounts to barely being a learner at all). 
There are other ways of seeing relationships between lenses and learners, and in particular of
seeing the entire gamut of learners (not just ones with trivial parameters) as lenses.

We first note the following.

\begin{lemma} \label{lem-learnalens}
    Every learner $(P,I,U,r)\colon A \to B$ is an asymmetric lens $(p,g)\colon P \times A \to B$.
\end{lemma}
\begin{proof}
    Let $g = I\colon P \times A \to B$ and let $p = \langle U,r \rangle\colon
    B \times (P \times A) \to P \times A$ be the unique function into the
    product $P \times A$ determined by $U$ and $r$.
\end{proof}

The resulting lenses $(\langle U,r \rangle, I)\colon P \times A \to B$ will
not in general be well-behaved. In particular the training process in
supervised learning would not usually be expected to satisfy PutGet. Whether
learners, when viewed as in this lemma as lenses should satisfy GetPut is a
subject of ongoing research, so for now we make no assumptions. We make a
comment again on this in Section~\ref{sec-Discussion}.

Of course, merely observing that learners are lenses in this way is not
especially useful if the composition of learners does not correspond to the
composition of lenses. And it cannot. Given two learners $(P,I,U,r)\colon A \to
B$ and $(P',I',U',r')\colon B \to C$ their corresponding lenses under the lemma
are not even composable, since they have types $(p,g)\colon P \times A \to B$
and $(p',g')\colon P' \times B \to C$.

As it happens, however, there is a Kleisli-like composition of these lenses that uses
the monoidal product in the category $\alens$ to convert the first of these
lenses, by taking its cartesian product with $P'$, to obtain a lens $P'\times
P \times A \to P' \times B$ which is then composable with the second lens.
Remarkably the resulting composition \emph{does} correspond precisely to the
composition of the original learners. But all this can be expressed better by
relating learners to certain \emph{symmetric} lenses, and we do that now.

\begin{lemma}\label{lem-LearnersAreSymmetricLenses}
    Every learner $(P,I,U,r)\colon A \to B$ is a symmetric lens 
  \[
    A \xleftarrow{(k,\pi_2)} P \times A \xrightarrow{(\langle U,r\rangle,I)} B
  \]
  with left leg a constant complement (and therefore well-behaved) lens.
\end{lemma}
\begin{proof}
  The right leg $(\langle U,r \rangle,I)$ is the asymmetric lens given in
  Lemma~\ref{lem-learnalens}, while the left leg $(k,\pi_2)$ is the constant
  complement (see Example~\ref{eg-ConstComp}) lens of the specified type.
\end{proof}

This gives the following correspondence:

\begin{center}
  \begin{tabular}{c|c}
    \textbf{Learner} $A \to B$ & \textbf{Symmetric lens} $A \to B$ \\\hline
    Hypotheses $P$ & Left leg complement $P$\\
    Implementation $I\colon P \times A \to B$ & Right leg Get $g\colon P \times A \to B$ \\
    Update $U\colon B \times P \times A \to P$ & Right leg Put $p\colon B \times P \times A \to P \times A$ (1st component) \\
    Request $r\colon B \times P \times A \to A$ & Right leg Put $p\colon B \times P \times A \to P \times A$ (2nd component)
  \end{tabular}
\end{center}

It should be remarked that in both the preceding lemmas we would normally use the word ``yields'' rather than ``is'', 
and we would be explicit about the process that converts a learner into a lens.  However, here we have used
``is'' to emphasise that what we are describing is nothing more than minor repackaging of the data.  Furthermore,
as the following theorem shows, the important interactions among the data (composition and monoidal product)
are exactly the same, whether one treats the data as learners or lenses.

We are now in a position to state the main result.

\begin{theorem} \label{thm-main}
    There is a faithful, identity-on-objects, symmetric monoidal functor
  $\lrn \to \slens$ mapping 
  \[
    (P,I,U,r)\colon A \to B
  \] to 
  \[
     A \xleftarrow{(k,\pi_2)} P \times A \xrightarrow{(\langle U,r\rangle ,I)} B.
  \]
\end{theorem}
%

With the correspondences established, the proof is largely a routine
verification. Since $\lrn$ and $\slens$ both have sets as objects, we may define
the functor to act as the identity on objects. On arrows, the functor acts as in
correspondence presented in Lemma~\ref{lem-LearnersAreSymmetricLenses} and
outlined in the table above; the fact that this operation is independent of
representive chosen follows immediately from the similarity between equivalence
relation conditions (E) and (E$'$) (see Definitions~\ref{def-equivslens} and
\ref{def-equivlearn}). It is easy to see that this proposed functor preserves
identities and the monoidal product.

The main difficulty is proving that the proposed functor preserves composition.
Before walking through this in detail on the next page, we first present a
useful lemma regarding composition of symmetric lenses.

Note that in Section~\ref{sec-SymmetricLenses} the composition of symmetric
lenses with left legs satisfying PutGet was presented in its maximum
generality. The motivation for that will be discussed in
Section~\ref{sec-Discussion}. To prove the above theorem, it will be helpful
to understand the simpler case where the left legs are known to be
constant complement. We summarise our notation for the composition in the following
diagram (in which notation is abused in the usual way for constant complement
lenses):
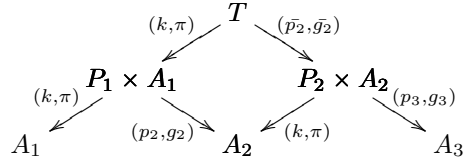
\[
\bfig
\Atriangle(400,250)/->`->`/<400,250>[T`{P_1\times A_1}`{P_2\times A_2};(k,\pi)`(\bar{p_2},\bar{g_2})` ]
\Vtriangle(400,0)/`->`->/<400,250>[{P_1\times A_1}`{P_2\times A_2}`A_2; `(p_2,g_2)`(k,\pi)]
\morphism(400,250)|a|<-400,-250>[{P_1\times A_1}`A_1;(k,\pi)]
\morphism(1200,250)|a|<400,-250>[{P_2\times A_2}`A_3;(p_3,g_3)]
\efig
\]
The following lemma gives an explicit formula for the composite of two symmetric lenses
whose left legs are constant complement.

\begin{lemma}\label{lem-RightLegs}
Suppose that  
$A_1 \xleftarrow{(k,\pi_2)} P_1 \times A_1 \xrightarrow{(p_2,g_2)} A_2$ and
$A_2 \xleftarrow{(k,\pi_2)} P_2 \times A_2 \xrightarrow{(p_3,g_3)} A_3$ are
spans of asymmetric lenses whose left legs, both denoted $(k, \pi_2)$, are constant complement lenses.
Note that constant complement lenses satisfy PutGet and hence we may compose
them using Definition~\ref{def-SymmetricLensComposition}. 

Their composite symmetric lens from $A_1$ to $A_3$ is represented by
\[
A_1 \xleftarrow{(k,\pi_3)} P_2 \times P_1 \times A_1 \xrightarrow{(p,g)} A_3,
\]
where $p$ and $g$ are given by
\[
p(a_3,(m_1,m_2,a_1)) = \big(m'_2  ,p_2(a'_2,(m_1,a_1)) \big),
\]
in which we let $(m_2',a_2') = p_3(a_3, (m_2,g_2(m_1,a_1))) \in P_2 \times A_2$, 
and
\[
g(m_2,m_1,a_1) = g_3(m_2,g_2(m_1,a_1)).
\]
(Note that to avoid overloading the notation $p_i$, we have written $m_i$ for
elements of $P_i$.)
\end{lemma}

\begin{proof}
Recall that in Definition~\ref{def-SymmetricLensComposition} we use notation as
in the following diagram.
$$
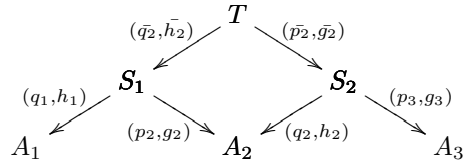
\bfig
\Atriangle(400,250)/->`->`/<400,250>[T`S_1`S_2;(\bar{q_2},\bar{h_2})`(\bar{p_2},\bar{g_2})` ]
\Vtriangle(400,0)/`->`->/<400,250>[S_1`S_2`A_2; `(p_2,g_2)`(q_2,h_2)]
\morphism(400,250)|a|<-400,-250>[S_1`A_1;(q_1,h_1)]
\morphism(1200,250)|a|<400,-250>[S_2`A_3;(p_3,g_3)]
\efig
$$
Consider then, as in Definition~\ref{def-SymmetricLensComposition}, the pullback $T$
in $\set$ of the cospan $P_1 \times A_1 \xrightarrow{g_2} A_2 \xleftarrow{\pi_2}
P_2 \times A_2$.  Knowing how to calculate pullbacks in $\set$, we may suppose
without loss of generality that the elements of $T$ are tuples $(m_2,m_1,a_1)$,
in which there is no $a_2$ explicitly mentioned since it must be equal to
$g_2(m_1,a_1)$, and that $\overline{h_2}$ and $\overline{g_2}$ are $\pi_{23}$ and
$P_2 \times g_2$ respectively.
More explicitly still, $T$ is just the product $P_2 \times P_1 \times A_1$,
$\overline{h_2}$ is the projection onto
$P_1 \times A_1$, and $\overline{g_2}$ is 
the arrow $P_2 \times P_1 \times A_1 \to P_2 \times A_2$ which preserves the $P_2$ value and 
uses $g_2$ to convert the other two values into an $A_2$ value.

According to Definition~\ref{def-SymmetricLensComposition}, the left leg of the composite is the 
composition of two asymmetric lenses denoted there as $(\overline{q_2},\overline{h_2})
\cp (q_1,h_1)$.
In the current context, $(q_1,h_1)$ is the constant complement lens $P_1 \times A_1 \to A_1$, and
we have seen the $\overline{q_2}$ may be taken to be the projection 
$P_2 \times P_1 \times A_1 \to P_1 \times A_1$.  Furthermore the definition of $\overline{h_2}$ in 
Definition~\ref{def-SymmetricLensComposition} is easily seen, up to reordering
of variables, to be 
in this context the constant complement Put.  
Finally, we have already seen in Example~\ref{eg-ConstComp} how the composition of two
constant complement lenses is a constant complement lens, so the left leg of the composition here
is simply the constant complement lens $(k,\pi_3): P_2 \times P_1 \times A_1 \to A_1$.

We turn now to the right hand leg $(p,g)$.  Again Definition~\ref{def-SymmetricLensComposition}
tells us that it is given by the asymmetric lens composition denoted there as
$(\overline{p_2},\overline{g_2}) \cp (p_3,g_3)$, in which $\overline{p_2}$ was defined by what
we referred to there as the ``somewhat complicated expression''.  Referring to 
Definition~\ref{def-AsymLensComposition} for how to compose asymmetric lenses, and using 
the notation of the statement of the lemma, we see that 
\[
p(a_3,(m_1,m_2,a_1)) = \overline{p_2}((m'_2,a'_2),(m_2,m_1,a_1)).
\]
We now show that this becomes, simply by substituting and simplifying,
$\big(m'_2,p_2(a'_2,(m_1,a_1))\big)$.

In detail
\begin{align*}
\overline{p_2}((m'_2,a'_2),(m_2,m_1,a_1)) &= \overline{p_2}(s'_2,(s_1,s_2))
\tag{in the notation of Definition~\ref{def-SymmetricLensComposition}}\\
 &= \big(p_2(h_2(s'_2),s_1),\, q_2(g_2(p_2(h_2(s'_2),s_1)),s'_2)\big)
 \tag{by Definition~\ref{def-SymmetricLensComposition}}\\
 &= \big(p_2(a'_2,s_1),  q_2(g_2(p_2(a'_2,s_1)),s'_2)\big) \tag{since $h_2(s'_2) =
 \pi_2(m'_2,a'_2) = a'_2$}\\
 &= \big(p_2(a'_2,s_1),  q_2(g_2(p_2(a'_2,s_1)),(m'_2,a'_2))\big) \tag{since $s'_2 = (m'_2,a'_2)$}\\
 &= \big(p_2(a'_2,s_1),  (m'_2, g_2(p_2(a'_2,s_1)))\big) \tag{since $q_2$ is constant complement}\\
 &= \big(p_2(a'_2,s_1),  m'_2 \big) \tag{by comment below} \\
 &= \big(p_2(a'_2,(m_1,a_1)),  m'_2\big) \tag{since $s_1 = (m_1,a_1)$}
\end{align*}
The second to last line holds since, as noted above, a fourth component in $T$
is superfluous since it has to be (and indeed is) $g_2$ applied to the first component.
Reordering the variables in that last line, because we have chosen to keep the $A_i$ in the last position,
completes the proof.
\end{proof}

We now return to the proof of Theorem~\ref{thm-main}.
\begin{proof}[Proof of Theorem~\ref{thm-main}]
  We will check that the proposed functor preserves composition. This is simply
  a matter of comparing the Get and Put of the right leg of composite symmetric
  learners, as detailed in Lemma~\ref{lem-RightLegs}, with the formulas for
  composition of learners as detailed in Definition~\ref{def-compositelearners}.

First compare the definition of $I \ast J$ in Definition~\ref{def-compositelearners} with the description of
the Get of the composite right legs, $g(m_2,m_1,a)$ of 
Lemma~\ref{lem-RightLegs}, recalling the correspondence between the implementation operations $I$ and $J$ and
the right leg Gets $g_2$ and $g_3$.  In other words, compare
$$
g(m_2,m_1,a) = g_3(m_2,g_2(m_1,a_1))
$$
with
$$
(I \ast J)(p,q,a) = J(q,I(p,a))
$$
noting the naming of variables means $p$ and $q$ correspond respectively to $m_1$ and $m_2$ (and that while
the order of parameters for $g$ is not important, the choice made for symmetric lens composition was to add
new parameters on the left, corresponding to the choice used in function composition).

Next, we compare the right leg Put $p$ of Lemma~\ref{lem-RightLegs} with the
composite update--request function $\langle U \ast V, r \ast s\rangle$. We do
this by considering each of the two components separately.

The $A_1$ component is $\pi_2(p_2(a'_2,(m_1,a_1)))$, which since $a'_2 =
\pi_2(p_3(a_3,(m_2,g_2(m_1,a_1))))$ is
\[
\pi_2 p_2\big( m_1,a_1,\pi_2 p_3(a_3,m_2,g_2(m_1,a_1))\big)
\] 
which should be compared with
\[
(r\ast s)(c,p,q,a) = r\big(p,a,s(c,q,I(p,a))\big)
\]
recalling again that $I$ corresponds to $g_2$,
that $p$ and $q$ correspond to $m_1$ and $m_2$,
that $a$ and $c$ correspond to $a_1$ and $a_3$, and 
that $r$ and $s$ correspond to $\pi_2p_2$ and $\pi_2p_3$
respectively.

Finally the $P_1 \times P_2$ component has, as its two coordinates, 
\[
\pi_1 p_2\big(\pi_2 p_3 (a_3,m_2,g_2(m_1,a_1)),m_1,a_1\big)
\hspace{.8cm}
\mbox{and}
\hspace{1.9cm}
\pi_1 p_3\big(a_3,m_2,g_2(m_1,a_1)\big)
\]
which should be compared respectively with
\[
U\big(s(c,q,I(p,a)),p,a\big)
\hspace{2.5cm}
\mbox{and}
\hspace{2.5cm}
    V\big(c,q,I(p,a)\big),
\]
recalling all the correspondences we've already pointed out, along with the correspondences
between $U$ and $V$ and $\pi_1 p_2$ and $\pi_1 p_3$ respectively.

These functions are all the same up to the specified renaming correspondences,
and hence our functor preserves composition.
\end{proof}

\section{Discussion}\label{sec-Discussion}

In this section we discuss two directions of research suggested by the main
theorem: laws for well-behaved learners, and links between learners and
multiary lenses.

\subsection{Learner laws}

We have shown that the usual notion of composition of symmetric lenses admits
a generalization that receives a functor from $\lrn$. Instead, one might
consider adding conditions to the notion of learner that permit learners to
embed into a more familiar notion of symmetric lens. Put another way: the
lens laws suggest analogues for learners.

For example, the GetPut law generalises as follows.

\begin{definition}
  We say that a learner $(P,I,U,r)$ obeys the \define{I-UR law} if for every
  parameter $p \in P$ we have both $r(I(p,a),p,a) = a$ and $U(I(p,a),p,a) = p$, or in string diagrams
          
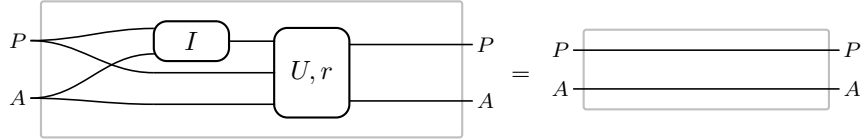
\[
            \begin{aligned}
              \begin{tikzpicture}[oriented WD]
                \node[bb port sep=1.5, bb={3}{2}] (R)     {$U,r$};
                \node[bb port sep=1, bb={2}{1}, left=.4 of R_in1]  (I)     {$I$};
                \coordinate (n2) at (I_in1|-R_in2);
                \coordinate (n3) at (I_in1|-R_in3);
                \node[ibb={2}{2}, fit=(R) (I)]           (outer) {};
                \node at ($(outer_in1')-(0.3,0)$) {\footnotesize $P$};
                \node at ($(outer_in2')-(0.3,0)$) {\footnotesize $A$};
                \node at ($(outer_out1'|-R_out1)+(0.3,0)$) {\footnotesize $P$};
                \node at ($(outer_out2'|-R_out2)+(0.3,0)$) {\footnotesize $A$};
                \draw (outer_in1) to (I_in1);
                \draw (outer_in2) to (I_in2);
                \draw (outer_in1) to (n2);
                \draw (n2) to (R_in2);
                \draw (outer_in2) to (n3);
                \draw (n3) to (R_in3);
                \draw (I_out1) to (R_in1);
                \draw (R_out1) to (outer_out1|-R_out1);
                \draw (R_out2) to (outer_out2|-R_out2);
              \end{tikzpicture}
            \end{aligned}
            =
            \begin{aligned}
              \begin{tikzpicture}[oriented WD]
                \node                           (I)     {};
                \node[ibb={2}{2}, fit=(I)]                          (outer) {};
                \node at ($(outer_in1')-(0.3,0)$) {\footnotesize $P$};
                \node at ($(outer_in2')-(0.3,0)$) {\footnotesize $A$};
                \node at ($(outer_out1'|-outer_in1)+(0.3,0)$) {\footnotesize $P$};
                \node at ($(outer_out2'|-outer_in2)+(0.3,0)$) {\footnotesize $A$};
                \draw (outer_in1) to (outer_out1|-outer_in1);
                \draw (outer_in2) to (outer_out2|-outer_in2);
              \end{tikzpicture}
            \end{aligned}
          \]
\end{definition}

This law asks that the lens $(\langle U,r\rangle,I)$ obey the GetPut law.
Intuitively, it states that if, at a given parameter $p$, the training
pair $(a,b)$ provided is already classified correctly by the learner---that
is, if the training pair is of the form $(a,I(p,a))$---then the update function
does not change the parameter and the request function does not request any
change to the input. This sort of property can be a desirable property of
learning algorithm, and a number of simple, important examples of learners,
including those of Example~\ref{eg-EuclideanLearners}, satisfy the I-UR law.

We have focussed more on the PutGet law in this paper. This may be generalised
as follows.

\begin{definition}
  We say that a learner $(P,I,U,r)$ obeys the \define{UR-I law} if for every
  parameter $p \in P$ and input $a \in A$ we have $I(U(b,p,a),r(b,p,a)) = b$, or
  in string diagrams
          
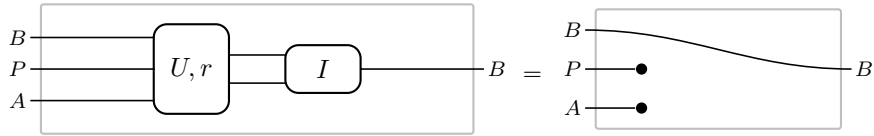
\[
            \begin{aligned}
              \begin{tikzpicture}[oriented WD]
                \node[bb port sep=1.5, bb={3}{2}] (R)     {$U,r$};
                \node[bb port sep=1.2, bb={2}{1}, right=.5 of R]  (I)     {$I$};
                \node[ibb={3}{1}, fit=(R) (I)]           (outer) {};
                \node at ($(outer_in1'|-R_in1)-(0.3,0)$) {\footnotesize $B$};
                \node at ($(outer_in2'|-R_in2)-(0.3,0)$) {\footnotesize $P$};
                \node at ($(outer_in3'|-R_in3)-(0.3,0)$) {\footnotesize $A$};
                \node at ($(outer_out1')+(0.3,0)$) {\footnotesize $B$};
                \draw (outer_in1|-R_in1) to (R_in1);
                \draw (outer_in2|-R_in2) to (R_in2);
                \draw (outer_in3|-R_in3) to (R_in3);
                \draw (R_out1|-I_in1) to (I_in1);
                \draw (R_out2|-I_in2) to (I_in2);
                \draw (I_out1) to (outer_out1);
              \end{tikzpicture}
            \end{aligned}
            =
            \begin{aligned}
              \begin{tikzpicture}[oriented WD]
                \node                           (I)     {};
                \node[ibb={3}{1}, fit=(I)]                          (outer) {};
                \node[circle, inner sep=1.5, fill] (n2) at ($(outer_in2)+(.5,0)$) {};
                \node[circle, inner sep=1.5, fill] (n3) at ($(outer_in3)+(.5,0)$) {};
                \node at ($(outer_in1')-(0.3,0)$) {\footnotesize $B$};
                \node at ($(outer_in2')-(0.3,0)$) {\footnotesize $P$};
                \node at ($(outer_in3')-(0.3,0)$) {\footnotesize $A$};
                \node at ($(outer_out1')+(0.3,0)$) {\footnotesize $B$};
                \draw (outer_in1) to (outer_out1);
                \draw (outer_in2) to (n2);
                \draw (outer_in3) to (n3);
              \end{tikzpicture}
            \end{aligned}
          \]
\end{definition}

This law asks that the lens $(\langle U,r\rangle,I)$ obey the PutGet law. The
intuition here is that when given a training pair $(a,b)$, if the requested
input $r(b,p,a)$ is given to the implementation function at the new parameter
$U(b,p,a)$, then the training pair $(r(b,p,a),b)$ will be correctly classified
by the learner. This is too strong for the incremental learning witnessed in
practical supervised learning algorithms such as neural networks. Nonetheless,
it is clear that a learner with this property would be in some sense desirable,
or well-behaved.

Indeed, learning algorithms in practice must take into account practical
considerations such as learning speed and convergence, and the prioritisation of
these considerations leads to methods that violate abstract properties such as
the I-UR and UR-I laws that might characterise what it means to learn
effectively. Nonetheless, we believe the formulation of these properties, from
well-motivated considerations such as our main theorem, suggest ideas that could
help frame and guide development of learning algorithms, especially should the
intent be to construct algorithms which can be reasoned about to some extent.

This view of learners obeying generalised lens laws suggests a view of a
learner as just a parametrised family of lenses, together with a rule for
choosing which lens in this family to use given some examples of what you
want the lens to do.

\subsection{Links with multiary lenses}

During the course of preparing this work for this workshop an interesting
similarity has come to light. In this paper the main tool we are using is
symmetric lenses with left leg constant complement, and right leg bare lenses.
In another paper presented at this workshop \cite{jrmml} that studies an
entirely different area, the main tool the authors use is symmetric lenses (in
fact, wide spans of lenses, so they may have in general more than two legs, but
they do have at least two legs) with left leg what is known as a closed
spg-lens, and right leg(s) arbritrary spg-lenses.

The similarity is more than just the linguistic parallel just described.
We'll not define spg-lenses or closed spg-lenses here, but we remark 
that constant complement lenses are indeed closed spg-lenses. In both cases
composition along those left legs is important, and the fact that they are,
in both cases, closed and satisfy PutGet is what is important for the
composition to work. The nature of the left legs is critical for the main
idea in both papers.

What of the right legs? At first they seem very different.  In this paper the right legs are bare lenses
--- they have a Put $p$ and Get $g$ and nothing else, neither more structure, nor axioms.
In apparent contrast, the right legs in \cite{jrmml} seem to have a substantial
amount of structure.  They do have a $p$ and a $g$, but they also have something
called an \emph{amendment} and several axioms.  The basic idea is that an
update, expressed there as an arrow $v$ in a category (here we only have the
codomain of such an arrow when we are doing an update because these are
set-based lenses), might result in not only a modification of the other
component (or in the case of the multiary lenses of \cite{jrmml}, the other
components), but also an \emph{amendment}. This amendment $a$ can be composed
with $v$ so that while the Get of the Put might not be $v$, it will be $av$.
In other words the amendment \emph{repairs} PutGet.

Now are the right hand lenses really that different?  There is a standard way of
seeing set based lenses as so called delta lenses (lenses that take arrows, not
just codomain objects, as the input for Put).  It appears as part of a unified
treatment of many different kinds of lenses \cite{jrusbdbebl} and involves
co-discrete categories.  In a codiscrete category there is a unique way of
extending each of the view updates from the bare lenses of this paper to make
them line up with their own Put, in other words a unique way of extending bare
lenses that satisfy no axioms to spg-lenses that satisfy PutGet.
So, the two ``main tools'' are actually remarkably similar.

They still differ in one respect (only):  spg-lenses are required to satisfy an
axiom that corresponds to GetPut here.  But we have just discussed how GetPut is
in fact a desirable property that might be asked of Learners.  If it were, then
the two very different projects are in fact using exactly alignable, novel,
tools: Closed amendment lenses (of the constant complement variety here) as left
leg and spg amendment lenses as right leg(s).

The similarities and what they might mean (if anything) will be considered further in future work.

\section{Conclusion}

To summarise, in this paper we have described a faithful, identity-on-objects,
symmetric monoidal functor from a category which captures the notion of
composable supervised learning algorithms to a suitable category of lenses.  To
do this, we presented a slight generalisation of the usual notion of symmetric
lens, in which we require a very weak form of well behavedness: a span of
asymmetric lenses in which the left leg satisfies the PutGet law. Despite the
general definition, these symmetric lenses still compose, and indeed equivalence
classes of them form the morphisms of a symmetric monoidal category $\slens$.
Our main theorem describes the aforementioned close, functorial relationship
between the category of learners, as defined in \cite{FST17}, with this category
of lenses.  In this theorem, we witness a surprising yet highly robust link
between two previously unrelated fields. We believe, as hinted by our brief
discussion in Section~\ref{sec-Discussion}, this to be a rich connection
deserving of further exploration.

\subsubsection*{Acknowledgements}

This work has been supported by the Australian Research Council and USA AFOSR
grants FA9550-14-1-0031, FA9550-17-1-0058. BF thanks Jules Hedges for first
bringing the contents of Remark~\ref{rem-triviallearner} to his attention.


\begin{thebibliography}{99}

  \bibitem
  [BS81]{bsusrv}
  Bancilhon, F. and Spyratos, N. (1981)
  \newblock Update semantics of relational views.
  \newblock \emph{ACM Trans. Database Syst.} \textbf{6}, 557--575.


  \bibitem[DXC11]{ddl}
  Diskin, Z., Xiong, Y. and Czarnecki, K. (2011)
  \newblock From State- to Delta-Based Bidirectional Model Transformations: the
  Asymmetric Case.
  \newblock \emph{Journal of Object Technology} \textbf{10}, 
  1--25.
  doi:10.5381/jot.2011.10.1.a6

  \bibitem[DKL18]{dml}
  Diskin, Z., K\"onig, H. and Lawford, M. (2018)
  \newblock Multiple model synchronization with multiary delta lenses.
  \newblock \emph{Lecture Notes in Computer Science}  \textbf{10802}, 21--37.


  \bibitem[FST19]{FST17} Fong, B., Spivak, D.\ I. and Tuy\'eras, R. (2019)
  \newblock Backprop as
  functor: a compositional perspective on supervised learning.
  \newblock To appear in \emph{Proceedings of the 34th Annual ACM/IEEE Symposium
  on Logic in Computer Science, LICS 2019}.
  \newblock Preprint available as arXiv:1711.10455.

  \bibitem[FS19]{FS19} Fong, B., and Spivak, D.\ I. (2019)
  \newblock \emph{An Invitation to Applied Category Theory: Seven Sketches in
  Compositionality},
  \newblock Cambridge University Press.

  \bibitem
  [FG+07]{fgmps}
  Foster, J., Greenwald, M., Moore, J.,  Pierce, B. and  Schmitt, A. (2007)
  \newblock Combinators for bi-directional tree transformations:
  A linguistic approach to the view update problem.
  \newblock \emph{ACM Transactions on Programming Languages and Systems} \textbf{29}.

  \bibitem
  [HPW11]{hpwsl}
  Hofmann, M., Pierce, B., and Wagner, D. (2011)
  \newblock Symmetric Lenses.
  \newblock \emph{ACM SIGPLAN-SIGACT Symposium on
  Principles of Programming Languages (POPL), 
  ACM SIGPLAN Notices} 
  \textbf{46}, 
  371--384.
  doi:10.1145/1925844.1926428

  \bibitem
  [JR12]{jrlpp}
  Johnson M. and Rosebrugh, R. (2012)
  \newblock Lens put-put laws: monotonic and mixed.
  \newblock Proceedings of the 1st International Workshop on Bidirectional Transformations, Tallin
  \emph{Electronic Communications of the EASST}, \textbf{49}, 13pp.


  \bibitem
  [JR13]{jrdl}
  Johnson, M. and Rosebrugh, R. (2013)
  \newblock Delta lenses and fibrations.
  \newblock Proceedings of the 2nd International Workshop on Bidirectional Transformations, Rome
  \emph{Electronic Communications of the EASST} \textbf{57}, 18pp.

  \bibitem
  [JR16]{jrusbdbebl}
  Johnson, M. and Rosebrugh, R. (2016)
  \newblock Unifying set-based, delta-based and edit-based lenses.
  \newblock Proceedings of the 5th International Workshop on Bidirectional Transformations, Eindhoven
  \newblock \emph{CEUR Proceedings} \textbf{1571}, 1--13.


  \bibitem
  [JR17]{jrjot}
  Johnson, M. and Rosebrugh, R. (2017)
  \newblock Symmetric delta lenses and spans of asymmetric delta lenses.
  \newblock \emph{Journal of Object Technology}, \textbf{16},  2:1--32.

  \bibitem[JR19]{jrmml}
  Johnson, M. and Rosebrugh, R. (2019)
  \newblock Multicategories of Multiary Lenses.
  \newblock To appear in \emph{Proceedings of the Eighth International Workshop
  on Bidirectional Transformations (Bx2019)}.

  \bibitem
  [PS03]{pslvut}
  Pierce, B. and  Schmitt, A. (2003)
  \newblock Lenses and view update translation. Preprint.

  \bibitem[Sel11]{Selinger} Selinger, P. (2011)
  \newblock A survey of graphical languages for
  monoidal categories.
  \newblock In Bob Coecke, editor, \emph{New Structures for
    Physics}, Lecture Notes in Physics 813:289--355, Springer.

\end{thebibliography}
\end{document}